\documentclass{article}

\usepackage[preprint]{neurips_2026}

\usepackage[utf8]{inputenc}
\usepackage[T1]{fontenc}
\usepackage{hyperref}
\usepackage{url}
\usepackage{booktabs}
\usepackage{amsfonts}
\usepackage{amsmath}
\usepackage{amssymb}
\usepackage{amsthm}
\usepackage{nicefrac}
\usepackage{microtype}
\usepackage{xcolor}
\usepackage{graphicx}
\usepackage{algorithm}
\usepackage{algorithmic}
\usepackage{cleveref}
\usepackage{enumitem}
\usepackage{multirow}

\newtheorem{theorem}{Theorem}
\newtheorem{corollary}[theorem]{Corollary}
\newtheorem{proposition}[theorem]{Proposition}
\newtheorem{definition}[theorem]{Definition}

\newcommand{\proj}{\text{proj}}
\newcommand{\TV}{\mathrm{TV}}
\newcommand{\KL}{D_{\mathrm{KL}}}
\newcommand{\E}{\mathbb{E}}

\newcommand{\calC}{\mathcal{C}}
\newcommand{\calL}{\mathcal{L}}
\newcommand{\calM}{\mathcal{M}}

\title{Future Validity is the Missing Statistic:\\From Impossibility to $\Phi$-Estimation\\for Grammar-Faithful Speculative Decoding}

\author{%
\textbf{Wenhua Nie \quad Zijie Meng \quad Kun Zou \quad Zheng Lin}\\
\textbf{Ziwei Li \quad Haoran Zheng \quad Jyh-Shing Roger Jang \quad Hao Zhang}\\[3pt]
\normalfont Correspondence: Wenhua Nie, National Taiwan University\\
\normalfont \texttt{d13944014@ntu.edu.tw}
}

\begin{document}
\maketitle

\begin{abstract}
Grammar-constrained generation is often combined with local vocabulary masking and speculative decoding, but the resulting sampling law is not the grammar-conditional distribution users usually intend.
We show that any speculative decoder with local mask access, Leviathan rejection, and rollback soundness samples from the locally projected distribution $\mu^{\proj}$ rather than the grammar-conditional distribution $\mu^\star$.
This extends the GAD impossibility result to speculative decoding; on Dyck grammars with Qwen3-8B, the total-variation gap can reach 0.996.
We identify the future-validity function $\Phi_t(y)=\Pr_p[\text{valid completion}\mid y]$ as the missing correction statistic.
The target distribution is a Doob transform of the base model with $h=\Phi$, while local masking corresponds to setting $h$ to one.
With exact $\Phi$, our oracle decoder FVO-Spec samples exactly from $\mu^\star$; with approximate $\Phi$, we bound the resulting total-variation error.
Because exact future validity is hard for general context-free grammars, we evaluate estimator hierarchies on tractable Dyck and finite JSON languages.
OneStep reduces Dyck TV by 14\% with under 1\% throughput overhead, exact dynamic programming reduces it by 97\%, and finite-language correction closes JSON gaps to numerical precision.
All fidelity claims are scoped to enumerable grammars and token tries.

\end{abstract}

\section{Introduction}
\label{sec:intro}

\begin{figure}[t]
\centering
\includegraphics[width=\linewidth]{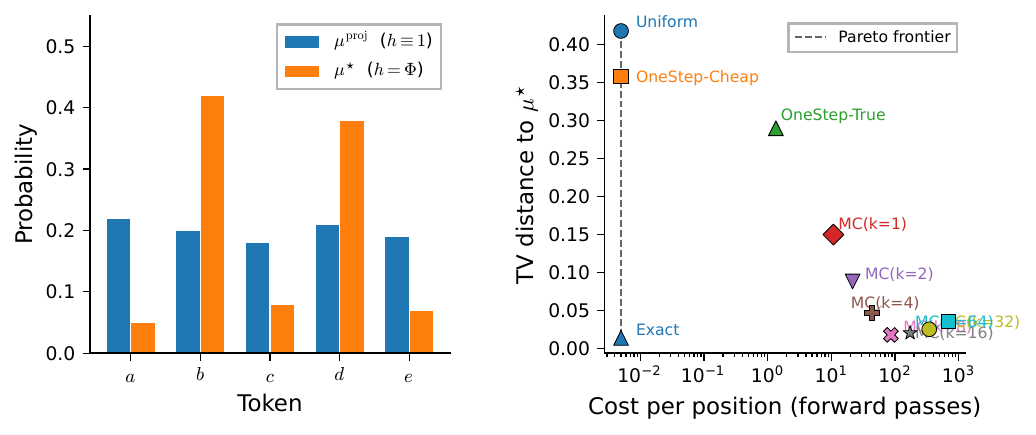}
\caption{The future validity gap and its correction. Local masking samples $\mu^{\proj}$, which diverges from $\mu^\star$ whenever future validity $\Phi_t$ varies across valid tokens; $\Phi_t$ reweighting recovers $\mu^\star$ exactly. On Dyck $D_{3,16}$, Uniform/OneStep/Exact reach TV $0.418/0.359/0.014$ to $\mu^\star$.}
\label{fig:hero}
\end{figure}

Grammar-constrained generation from large language models is now a foundational primitive in production systems. When an LLM serves a JSON API, translates natural language to SQL, or synthesizes code that must parse, the output must conform to a formal grammar. Frameworks such as XGrammar~\citep{dong2024xgrammar} and Outlines~\citep{willard2023outlines} enforce this via \emph{local vocabulary masking}: at each decoding step, the logit vector is masked to retain only tokens consistent with the grammar prefix, then renormalized. To amortize the cost of large target models, this masking is routinely combined with speculative decoding~\citep{leviathan2023fast,chen2023accelerating}, the dominant paradigm for lossless LLM inference acceleration now integrated into vLLM~\citep{kwon2023vllm}, SGLang~\citep{zheng2024sglang}, and other serving stacks. The resulting pipeline---local mask plus speculative verification---is deployed at scale, and users implicitly trust that the output distribution faithfully represents the language model's beliefs conditioned on grammaticality.

But what distribution does this pipeline \emph{actually} sample? We prove that it samples from the \emph{locally projected} distribution $\mu^{\proj}$---the product of per-step masked-and-renormalized conditionals---rather than from the true grammar-conditional distribution $\mu^\star$, which is the language model's full sequence distribution restricted to the grammar's language and renormalized. The distinction is not merely theoretical. On Dyck languages $D_{3,16}$ (balanced parentheses with maximum nesting depth 3 and length 16), the total variation distance $\TV(\mu^{\proj}, \mu^\star)$ reaches \textbf{0.996} under Qwen3-8B: the locally projected distribution places nearly all its mass on sequences that $\mu^\star$ assigns negligible probability to, and vice versa. On finite canonical JSON languages with 3--24 valid strings, the TV gap ranges from 0.17 to 0.68 under Qwen3-8B. Thus, whenever future validity varies substantially across locally valid tokens, the standard local-mask pipeline can produce \emph{systematically biased} structured outputs---sequences that satisfy the grammar but are drawn from the wrong distribution. Fields that require calibrated uncertainty over structured outputs (Bayesian program synthesis, probabilistic database queries, grammar-constrained reasoning chains) are directly affected.

What is missing from the standard pipeline? A single quantity: the \emph{future validity function} $\Phi_t(y) = \Pr_{z \sim p}[x_{<t}\, y\, z_{t+1:T} \in \calL(\calC)]$, the probability under the base model $p$ that appending token $y$ to the current prefix leads to a string that can be completed to a valid member of the grammar's language. The true conditional $\mu^\star_t(y)$ is proportional to $p_t(y) \cdot \Phi_t(y)$---this is precisely a Doob $h$-transform~\citep{doob1984classical} of $p$ with harmonic function $h = \Phi$. Local vocabulary masking corresponds to the degenerate approximation $\Phi \equiv 1$: it treats every locally valid token as equally likely to lead to a valid completion, ignoring the massive variation in future validity that recursive grammars induce. The gap between $\mu^{\proj}$ and $\mu^\star$ is entirely controlled by the non-uniformity of $\Phi_t$ across the valid token set $A_t$.

\paragraph{Contributions.} We develop a theoretical and empirical framework for diagnosing distributional bias in grammar-constrained speculative decoding and correcting it where future validity can be estimated or compiled.

\begin{enumerate}[leftmargin=*,nosep]

\item \textbf{LMS Diagnostic} (\Cref{sec:lms}, Proposition~\ref{thm:lms}). We formalize three axioms---local mask access, Leviathan rejection, and rollback soundness---that characterize grammar-constrained speculative decoders (the LMS class), and derive that any system satisfying these axioms has per-step sampling kernel $\mu^{\proj}_t$. This instantiates the GAD observation~\citep{park2024grammar} in the speculative setting by making the masked-target assumption explicit and identifying the axioms that force the bias. We validate the diagnostic across 33 grammar--model settings confirming exact distributional match between empirical samples and the $\mu^{\proj}$ prediction.

\item \textbf{$\Phi_t$ Framework} (\Cref{sec:phi}). We identify future validity as the exact per-token reweighting factor from $\mu^{\proj}$ to $\mu^\star$, give an oracle FVO-Spec verifier that recovers $\mu^\star$, prove a KL identity in terms of $\Phi$-concentration, and show that local masking is exact iff $\Phi_t$ is constant over $A_t$.

\item \textbf{$\Phi$-Approximation Error Bound} (\Cref{sec:phi}, Theorem~\ref{thm:phi_approx}). If $\sup_{y \in A_t} |\hat{\Phi}_t(y)-\Phi_t(y)|\leq\delta$ and $\bar{\Phi}_t>\delta$, then $\TV(\hat{\mu}_t,\mu^\star_t)\leq \delta/(\bar{\Phi}_t-\delta)$. We also give a relative-error condition, Monte Carlo sample-complexity estimates, and a \#P-hardness result for exact $\Phi_t$ on general CFGs.

\item \textbf{Estimator Hierarchy and Diagnostic Barrier} (\Cref{sec:estimators}). Motivated by hardness, we develop four tiers spanning the cost--fidelity Pareto frontier:
\begin{itemize}[nosep]
\item \emph{Uniform} ($\hat{\Phi} \equiv 1$): recovers $\mu^{\proj}$, zero cost---the status quo.
\item \emph{OneStep}: one-step grammar lookahead using materialized logits and trie queries, zero neural forwards, $<$0.5ms overhead. It reduces TV by 14\% on Dyck $D_{3,16}$ but can worsen finite JSON schemas.
\item \emph{MC}: Monte Carlo rollout with adjustable cost--fidelity tradeoff and Hoeffding-style estimates.
\item \emph{Exact}: dynamic programming enumeration for tractable grammar subclasses (Dyck, finite, regular). Achieves 97\% TV reduction.
\end{itemize}
We therefore present the hierarchy as a diagnosis and design space rather than as a finished production estimator.

\item \textbf{Exact Oracle Validation} (\Cref{sec:experiments}). On Dyck $D_{3,16}$, Uniform/OneStep/Exact reach TV $0.418/0.359/0.014$ to $\mu^\star$. On Qwen3-8B finite canonical JSON up to 2{,}000 strings, exact token-trie $\Phi$ correction reduces TV $0.174$--$0.681$ to numerical zero, and a finite-trie FVO-Spec loop samples within TV $0.0052$ of $\mu^\star$ on A--D with mean acceptance $0.947$.

\end{enumerate}

The remainder of the paper is organized as follows. \Cref{sec:lms} formalizes the LMS class and records the kernel characterization. \Cref{sec:phi} develops the $\Phi_t$ framework, including the Doob $h$-transform perspective, oracle achievability, KL identity, and fidelity bound. \Cref{sec:estimators} presents the estimator hierarchy with cost and error analysis. \Cref{sec:experiments} reports experimental validation, and \Cref{sec:discussion} discusses limitations and future directions.

\section{Related Work}
\label{sec:related}

\paragraph{Speculative Decoding.}
Speculative decoding was introduced by \citet{leviathan2023fast} and \citet{chen2023accelerating} as lossless LLM acceleration via draft-then-verify. Variants include tree-structured verification~\citep{miao2024specinfer}, hidden-state prediction~\citep{li2024eagle,li2025eagle3}, multi-head drafting~\citep{cai2024medusa}, and block diffusion~\citep{chen2026dflash}. We study the distributional properties of the speculate-verify loop when grammar constraints are imposed, showing that the Leviathan rejection mechanism preserves whatever target distribution it is given---but local masking gives it the wrong one.

\paragraph{Grammar-Constrained Generation.}
XGrammar~\citep{dong2024xgrammar}, XGrammar-2~\citep{dong2026xgrammar2}, and Outlines~\citep{willard2023outlines} enforce structural constraints via per-step vocabulary masking. Grammar-Aligned Decoding (GAD)~\citep{park2024grammar} showed that local renormalization distorts the model's true conditional over valid strings, proposing sequential Monte Carlo (ASAp) to approximate expected futures during autoregressive decoding. Our $\Phi_t$ is the same expected-future quantity; the distinction is the \emph{operational target} of speculative serving systems. GAD changes the decoder; we ask what distribution is sampled by the deployed local-mask + Leviathan verifier pipeline, show that the verifier preserves $\mu^{\proj}$ unless its target is $\Phi$-reweighted, and quantify how approximate $\Phi$ estimates perturb the verifier-loop distribution. We view this as the SD-specific instantiation of the GAD diagnostic, with LMS axioms, fixed-budget error accounting, DFlash verifier-path tests, and throughput-compatible estimator diagnostics as the added obligations.

\paragraph{Doob $h$-Transforms in Generation.}
Related conditional-generation perspectives appear in score-based diffusion models~\citep{song2020score} and energy-based controlled text generation~\citep{qin2022cold}. Concurrently, \citet{chen2026attention} characterize hard-masked decoding distortion via a Doob $h$-transform and derive one-step KL/TV bounds in terms of survival probability, focusing on the computational complexity of the masking engine itself. Our work differs in three respects: (i) we instantiate the $h$-transform analysis for \emph{speculative decoding} and identify exactly where the verify loop must change its target distribution; (ii) we develop an estimator hierarchy (Uniform $\to$ OneStep $\to$ MC $\to$ Exact) with cost--fidelity accounting and both additive and relative-error certificates; and (iii) we validate the mechanism in real-model finite JSON and DFlash verifier-path experiments. We do not claim that the Doob perspective itself is unique to this paper.

\paragraph{Adaptive Acceptance in Speculative Decoding.}
Recent work has explored relaxing the fixed acceptance threshold in speculative decoding. EARS~\citep{wang2025ears} adjusts the threshold using $1 - \max(p_T)$ as a measure of target uncertainty, while EASD~\citep{li2025easd} uses entropy-based penalties in the acceptance decision. These methods trade distributional exactness for throughput by relaxing acceptance at high-uncertainty positions. In contrast, our $\Phi$-reweighting approach aims to \emph{correct} the distribution toward the true conditional $\mu^\star$ rather than further relaxing fidelity guarantees.

\paragraph{Prefix Probabilities and Inside Algorithms.}
Computing the probability of valid completions under CFGs is related to classical inside-probability computations~\citep{stolcke1995efficient} and semiring parsing~\citep{goodman1999semiring}. Our $\Phi_t$ is closely related to the backward or inside probability in probabilistic parsing: the probability that a partial derivation can be completed. The key distinction is that $\Phi_t$ conditions on an autoregressive LM $p$ rather than a PCFG, making it dependent on the base model and generally intractable. The \#P-completeness of counting valid strings~\citep{gore1997counting} underlies this hardness and motivates our estimator hierarchy.

\section{The LMS Characterization}
\label{sec:lms}

We first formalize what distribution grammar-constrained speculative decoders actually sample. This extends the observation of \citet{park2024grammar} from standard autoregressive decoding to the speculate-verify-resample setting.

\subsection{Setup}

Let $V$ be a finite vocabulary (including a distinguished EOS token) and $p$ a target language model with $p_t(y) = p(y \mid x_{<t})$. We work with sequences of length at most $T$ terminated by EOS; all definitions extend naturally to variable-length generation by treating EOS as a valid terminal transition. Let $\calC$ be a prefix-checkable constraint inducing valid-next-token sets $A_t = A_\calC(x_{<t}) \subseteq V$. We assume the non-degenerate case: for every reachable prefix, $A_t$ is non-empty and $Z_p=\sum_{y\in A_t}p_t(y)>0$, and the grammar has positive total model mass $Z_\calC>0$. Recall the two target distributions:

\begin{definition}[Locally Projected Target]
$\mu^{\proj}_t(y) = p_t(y) \cdot \mathbf{1}[y \in A_t] / Z_p$, where $Z_p = \sum_{y' \in A_t} p_t(y')$.
\end{definition}

\begin{definition}[True Grammar-Conditional Target]
$\mu^\star(x_{1:T}) = p(x_{1:T}) \cdot \mathbf{1}[x_{1:T} \in \calL(\calC)] / Z_\calC$, where $Z_\calC = \sum_{y \in \calL(\calC)} p(y)$.
\end{definition}

\subsection{The LMS Class}

A speculative decoding method $\calM$ belongs to class \textbf{LMS} (Leviathan-type Local-Mask Speculative) when it satisfies three axioms:

\textbf{B1} (Local Mask): At each step $t$, $\calM$ accesses the constraint only through $A_t$.

\textbf{B2} (Leviathan Rejection): Accept/reject uses the standard rule on masked distributions $\tilde{p}_t$ and $\tilde{q}_t$.

\textbf{B3} (Rollback Soundness): Exact constraint state is maintained via accept/rollback.

These axioms capture the standard pipeline used by XGrammar and Outlines when combined with Leviathan rejection sampling.

\begin{proposition}[LMS Characterization]
\label{thm:lms}
For every $\calM \in$ LMS, the per-step sampling kernel on committed prefix $x_{<t}$ is $\mu^{\proj}_t$. The marginal law of any sample is $\mu^{\proj}$.
\end{proposition}

The proof follows directly from Leviathan's Theorem~1~\citep{leviathan2023fast}: under B1--B3, the rejection step operates on $\tilde{p}_t = \mu^{\proj}_t$, and each committed token is distributed accordingly. Thus the contribution of \Cref{thm:lms} is a negative characterization of the standard constrained-SD pipeline and its forced sampling target, not a new rejection-sampling theorem. It rules out a common but incorrect inference: because Leviathan verification is lossless for its supplied target, adding a grammar mask does not make the target the grammar-conditional law $\mu^\star$. The chain rule gives the full-sequence identity. Full proof in \Cref{app:proof_lms}.

\begin{corollary}[Impossibility]
\label{cor:impossibility}
For every $(p, \calC)$ with $\mu^{\proj} \neq \mu^\star$, no LMS method samples $\mu^\star$.
\end{corollary}

\textbf{Witness.} For Qwen3-8B with Dyck $D_{3,16}$ (5{,}000 samples, temperature 1), $\TV(\mu^{\proj}, \mu^\star) = 0.996 \pm 0.002$ (95\% CI: $[0.994, 0.998]$). The near-maximal TV arises because Qwen3-8B assigns negligible probability to valid balanced-parenthesis continuations under unmasked sampling ($\Phi_\text{root} \approx 10^{-7}$), so $\mu^\star$ concentrates on high-$\Phi$ tokens that $\mu^{\proj}$ systematically underweights. A later analytic bounded-Dyck confirmation under the bracket-generation prompt gives $\TV(\mu^{\proj},\mu^\star)=0.999948\pm0.000002$ on the same support (\Cref{sec:experiments}), so both protocols show a near-maximal gap. On finite canonical JSON languages with 3--24 valid strings, Qwen3-8B yields $\TV(\mu^{\proj},\mu^\star)=0.174$--$0.681$ (\Cref{sec:experiments}).

\paragraph{Empirical validation.} We verify \Cref{thm:lms} across 33 settings: 4 architectures (Qwen2.5, Qwen3, Llama-3.1, Vicuna), 2 grammar families (Dyck, regex). All show $\TV(\text{empirical}, \mu^{\proj}) \leq 0.007$ while $\TV(\mu^{\proj}, \mu^\star)$ ranges from small regex gaps to 0.996 on Qwen3-8B Dyck. See \Cref{tab:lms_val} for details.

\section{Future Validity as Sufficient Statistic}
\label{sec:phi}

The LMS characterization shows that standard grammar-constrained speculative decoders sample from the local projection $\mu^{\proj}$.
The missing quantity is the \emph{future validity} of each locally valid token:
\[
\Phi_t(y\mid x_{<t})=\Pr_{z\sim p}[x_{<t}yz_{t+1:T}\in\mathcal{L}(\mathcal{C})].
\]
The true grammar-conditional kernel is the Doob $h$-transform
\[
\mu^\star_t(y\mid x_{<t})=\frac{p_t(y)\Phi_t(y\mid x_{<t})}{\sum_{y'\in A_t}p_t(y')\Phi_t(y'\mid x_{<t})}.
\]

\begin{proposition}[Exactness Condition]
\label{prop:exact}
$\mu^{\proj}_t=\mu^\star_t$ if and only if $\Phi_t$ is constant over locally valid tokens.
\end{proposition}
Thus local masking is safe for permissive constraints but biased for recursive grammars where one token can sharply reduce the probability of any valid continuation.

\begin{proposition}[Oracle FVO-Spec]
\label{prop:oracle}
A speculative decoder that performs rejection against the $\Phi$-reweighted target samples from $\mu^\star_t$ at every step.
\end{proposition}

\paragraph{Gap identity.}
The per-step divergence has the closed form
\[
\KL(\mu^\star_t\|\mu^{\proj}_t)=\mathbb{E}_{Y\sim\mu^\star_t}\!\left[\log\frac{\Phi_t(Y)}{\bar{\Phi}_t}\right],
\]
where $\bar{\Phi}_t=\mathbb{E}_{\mu^{\proj}_t}[\Phi_t]$.

\begin{theorem}[$\Phi$-Approximation Fidelity]
\label{thm:phi_approx}
If $|\hat{\Phi}_t(y)-\Phi_t(y)|\le\delta$ for all $y\in A_t$ and $\delta<\bar{\Phi}_t$, then
\[
\TV(\hat{\mu}_t,\mu^\star_t)\le\frac{\delta}{\bar{\Phi}_t-\delta}.
\]
\end{theorem}
The bound explains why recursive grammars are difficult: when $\bar{\Phi}_t$ is small, only highly accurate $\Phi$ estimators give non-vacuous fidelity guarantees.

\begin{proposition}[Hardness]
\label{prop:hardness}
Computing $\Phi_t(y)$ exactly is \#P-hard for general context-free grammars, even under a unigram base language model.
\end{proposition}
The appendix gives the reduction from \textsc{CountCFG} and the corollaries for local projection error and Monte Carlo sample complexity.

\section{A Hierarchy of $\Phi_t$ Estimators}
\label{sec:estimators}

The hardness result (\Cref{prop:hardness}) implies that practical grammar-faithful speculative decoding must approximate $\Phi_t$. We develop a hierarchy of estimators spanning the cost--fidelity Pareto frontier.

\paragraph{Critical invariant.} $\hat{\Phi}_t$ must be a \emph{per-candidate vector}, not a scalar. Adding a constant to all logits before softmax is a no-op by shift invariance. An estimator that produces $\hat{\Phi}_t(y) = c$ for all $y$ yields $\hat{\mu}_t = \mu^{\proj}_t$ regardless of $c$, providing no correction. We enforce this invariant in all implementations (see tests in supplementary).

\paragraph{Certification target.} \Cref{thm:phi_approx} certifies estimators with
small additive error; \Cref{cor:relative_phi} certifies estimators that preserve
the multiplicative scale of future validity. Exact DP has zero error under both
certificates. By contrast, Uniform and OneStep can have unbounded relative error
on dead-end or near-dead-end candidates, explaining why our experiments report
the additive certificate and why robust relative-error estimation remains an
open design target rather than a claim of the current OneStep implementation.

\subsection{Tier~0: Uniform (Baseline)}

Setting $\hat{\Phi}_t \equiv 1$ recovers $\mu^{\proj}$ with zero cost. This is the status quo for XGrammar~\citep{dong2024xgrammar}, Outlines~\citep{willard2023outlines}, and all LMS methods characterized in \Cref{sec:lms}. Its additive error is $\delta = 1 - \Phi_{\min}$, which reaches $1 - 10^{-6}$ on adversarial Dyck grammars.

\subsection{Tier~1: OneStep Lookahead}

\begin{definition}[OneStep Estimator]
\label{def:onestep}
\begin{equation}
\hat{\Phi}^{\text{1s}}_t(y) \;=\; \sum_{u \in A_{t+1}(x_{<t} y)} p_t(u)
\end{equation}
where $A_{t+1}(x_{<t} y)$ is the valid-next-token set after appending $y$, obtained from the grammar matcher via trie lookup, and $p_t(u)$ are the target logits at position $t$ (available from verification).
\end{definition}

This is a \emph{heuristic} approximation: it uses $p_t(\cdot) = p(\cdot \mid x_{<t})$ as a proxy for $p(\cdot \mid x_{<t} y)$, a context-independence assumption that introduces an uncharacterized error term beyond the idealized analysis below. The approximation is reasonable when the prefix $x_{<t}$ is long relative to a single token, but can fail on highly concentrated schemas where one token radically changes the valid continuation set (see E3). The grammar query $A_{t+1}(x_{<t} y)$ is the \emph{discriminative} component, computed in $O(|V|)$ via trie traversal without any neural forward passes.

\begin{proposition}[OneStep Error Characterization]
\label{prop:onestep}
Under the exact conditioning $p(\cdot \mid x_{<t} y)$, the true one-step estimator satisfies:
\begin{equation}
\hat{\Phi}^{\mathrm{true}}_t(y) - \Phi_t(y) = \sum_{u \in A_{t+1}(x_{<t}y)} p(u \mid x_{<t}y) \cdot \bigl(1 - \Phi_{t+1}(u \mid x_{<t}yu)\bigr)
\label{eq:onestep_error}
\end{equation}
This is the probability-weighted average of future invalidity over valid next tokens. It is:
\begin{itemize}[nosep]
\item \textbf{Small} when future validity is high: for grammars where most locally valid tokens lead to completable paths ($\Phi_{t+1}(u) \approx 1$), the error is controlled by the small fraction of near-dead-end tokens. This occurs for permissive regular grammars when the LM concentrates mass on valid continuations.
\item \textbf{Large} for recursive grammars where locally valid tokens frequently lead to dead ends ($\Phi_{t+1} \ll 1$), e.g., Dyck languages near the depth or length limit where most locally valid bracket choices exhaust the remaining budget.
\end{itemize}
\end{proposition}

\paragraph{Cost analysis.} The grammar trie query for $A_{t+1}(x_{<t}y)$ costs $O(|V|)$ per candidate. Summing probabilities from already-materialized logits adds negligible compute. For $|A_t|$ candidates, total overhead is $O(|A_t| \cdot |V|)$ trie operations with zero neural forward passes. In our implementation, this adds $<$0.5ms per position on typical JSON schemas.

\subsection{Tier~1c: Balanced-Gated Learned Value Estimator}

OneStep is low-cost but brittle because it has no memory of deeper completion
asymmetries. We therefore include a small amortized estimator as a controlled
proof of concept. For each finite-trie edge $(x_{<t},y)$, we train a two-layer
MLP to predict $\log \Phi_t(y)$ from finite-trie features: prefix depth, valid-branch
counts, child-subtree counts, remaining-length statistics, target token
probability, valid-mass fraction, and the OneStep value. Training labels are
exact $\Phi$ values computed on other schemas; evaluation is
leave-one-schema-out, so the held-out schema's exact $\Phi$ labels are not used
for fitting.

To avoid negative transfer on schemas where OneStep is already essentially
exact, we add a non-oracle structural gate. If all valid children at a prefix
have identical trie statistics (same subtree count, remaining-length summary,
branch count, and terminal flag), the estimator falls back to OneStep; otherwise
it uses the learned value. The gate decision uses only grammar topology
(subtree counts and remaining-length statistics), not logits, terminal-law TV,
or exact held-out $\Phi$ labels; the OneStep fallback to which it routes uses
LM logits. On the finite JSON benchmark, it preserves the balanced type-value
schema where OneStep is at numerical floor while allowing learned correction on
the concentrated schemas.
We treat this as evidence that amortized $\Phi$ estimation is possible on finite
tries, not as a runtime benchmark or a production estimator for arbitrary
xgrammar schemas. We also tested histogram-gradient-boosted regression in this
prototype; without the balanced gate it shows the same type-value failure mode,
so the reported result uses the gated MLP configuration.

\subsection{Tier~2: Monte Carlo Rollout}

For $k$ independent rollouts of $h$ tokens each:
\[
\hat{\Phi}^{\text{mc}}_t(y) = \frac{1}{k}\sum_{j=1}^k \mathbf{1}\bigl[\text{rollout}_j(x_{<t}y) \in \calL(\calC)\bigr]
\]
By Hoeffding's inequality, $|\hat{\Phi}^{\text{mc}}_t(y) - \Phi_t(y)| \leq \sqrt{\ln(2/\alpha)/(2k)}$ with probability $\geq 1-\alpha$ per candidate. Applying \Cref{thm:phi_approx} with a union bound over $|A_t|$ candidates gives:
\[
\TV(\hat{\mu}_t, \mu^\star_t) \leq \frac{\sqrt{\ln(2|A_t|/\alpha)/(2k)}}{\bar{\Phi}_t - \sqrt{\ln(2|A_t|/\alpha)/(2k)}}
\]
with probability $\geq 1-\alpha$. Each rollout requires one target forward pass per step, giving total cost $O(k \cdot h \cdot |A_t|)$ forwards. This is prohibitive for online use but serves as a high-fidelity reference for validating cheaper estimators.

\subsection{Tier~3: Exact Enumeration (Oracle)}

For tractable grammar subclasses, $\Phi_t$ can be computed exactly via dynamic programming:
\begin{itemize}[nosep]
\item \textbf{Dyck $D_{d,L}$}: backward DP over all valid prefixes up to depth $d$ and length $L$. Cost is linear in the enumerated prefix graph.
\item \textbf{Finite languages}: exhaustive enumeration over $\calL(\calC)$. Cost $O(|\calL|)$.
\item \textbf{Regular languages}: single automaton step per candidate. Cost $O(|A_t|)$.
\end{itemize}
Exact computation enables oracle validation of the fidelity bound (\Cref{thm:phi_approx}), providing ground truth that is unavailable for general CFGs.

\paragraph{Practical scope.} Today, exact $\Phi$ is deployment-feasible only on
finite or regularized schemas whose state graph can be enumerated or cached
(e.g., enum-heavy JSON, finite code systems, bounded templates). Bounded-depth
grammars admit DP only when the bound is small enough. General CFG-style JSON
with arbitrary nesting or long free-text fields requires approximation; our
DFlash/xgrammar pilot tests the verifier mechanism in that regime, but not an
end-to-end exact terminal-law correction.

\subsection{Budget-Gated Hybrid Sketch}

For deployment, the hierarchy can be composed via a two-stage gate. \textbf{Stage~1} (pre-trigger): compute $\hat{\Phi}^{\text{1s}}$ and measure its dispersion $\log(\hat{\Phi}_{\max}/\hat{\Phi}_{\min})$. If dispersion exceeds a threshold $\tau_d$ and the position has high entropy, escalate to Stage~2. \textbf{Stage~2} (post-trigger): apply MC rollouts with budget $k$, then check whether the TV impact exceeds a minimum threshold. If not, fall back to OneStep. This avoids MC overhead at easy positions where OneStep suffices. We treat this gate as an engineering sketch rather than a headline empirical result; the experiments below evaluate the individual estimator tiers so that cost and fidelity are attributable.

\section{Experiments}
\label{sec:experiments}

We validate the framework on tractable grammars where exact $\Phi_t$ is computable.
Experiments use Dyck languages, permissive regex constraints, and finite JSON schemas; TV distances are estimated from 10{,}000 samples with bootstrap confidence intervals.

\subsection{LMS Characterization}

\begin{table}[t]
\centering
\caption{LMS validation: empirical sampling distribution matches $\mu^{\proj}$ across architectures and grammars.}
\label{tab:lms_val}
\begin{tabular}{llccc}
\toprule
Architecture & Grammar & Settings & Max residual & Mean residual \\
\midrule
Qwen2.5 & Dyck CF & 8 & 0.005 & 0.003 \\
Qwen3 & Dyck CF & 8 & 0.007 & 0.004 \\
Llama-3.1 & Dyck CF & 6 & 0.004 & 0.002 \\
Vicuna & Dyck CF & 3 & 0.006 & 0.004 \\
Qwen2.5 & Regex & 4 & 0.005 & 0.003 \\
Qwen3 & Regex & 4 & 0.003 & 0.002 \\
\bottomrule
\end{tabular}
\end{table}
All residuals are at most 0.007, while $\TV(\mu^{\proj},\mu^\star)$ ranges up to 0.9976, confirming that the observed bias is the predicted local-projection bias rather than sampling noise.

\subsection{Estimator Hierarchy}

\begin{table}[t]
\centering
\caption{$\Phi$ estimator hierarchy on Dyck $D_{3,16}$. Better $\Phi$ estimates reduce TV to the true conditional distribution.}
\label{tab:hierarchy}
\begin{tabular}{lccc}
\toprule
Estimator & TV to $\mu^\star$ & $\delta$ & Cost/pos \\
\midrule
Uniform & 0.418 & 0.986 & 0 \\
OneStep-Cheap & 0.359 & 0.861 & 0 \\
OneStep-True & 0.290 & 0.825 & 1.4 \\
MC($k$=4) & 0.047 & 0.763 & 43.2 \\
MC($k$=8) & 0.018 & 0.565 & 86.4 \\
Exact oracle & 0.014 & 0.000 & DP \\
\bottomrule
\end{tabular}
\end{table}
The hierarchy reduces TV from 0.418 to 0.014, a 97\% reduction.
OneStep-Cheap gives a free 14\% reduction, while exact dynamic programming closes the gap on tractable grammars.

\subsection{Bias and Throughput}

The $\mu^{\proj}$--$\mu^\star$ gap produces concrete structural distortion: on Dyck $D_{3,16}$, local projection overproduces deep strings (mean nesting depth 1.72 versus 1.09 under $\mu^\star$) and longer strings (mean length 5.85 versus 2.77).
Exact $\Phi$ correction recovers both distributions.
A cost model shows that OneStep-Cheap preserves SD throughput while improving fidelity, whereas MC rollouts trade speed for lower TV.
Full grammar-family, bound-tightness, and speed--fidelity plots are moved to the appendix.

\subsection{Grammar Family and Structural Bias}

\begin{table}[t]
\centering
\caption{Grammar-family comparison. The $\mu^{\proj}$--$\mu^\star$ gap varies widely across constraint types; exact $\Phi$ closes the gap when sampling variance is moderate.}
\label{tab:grammar_family}
\begin{tabular}{lcccc}
\toprule
Grammar & $\TV(\mu^{\proj}, \mu^\star)$ & $\bar{\Phi}$ & Exact TV & OneStep TV \\
\midrule
JSON status & 0.575 & 0.101 & $<$0.001 & 0.622 \\
JSON action & 0.820 & 0.069 & 0.097 & 0.843 \\
JSON method-path & 0.847 & 0.066 & 0.601 & 0.878 \\
Dyck $D_{3,12}$ & 0.456 & 0.516 & 0.018 & 0.419 \\
Dyck $D_{3,16}$ & 0.418 & 0.467 & 0.014 & 0.359 \\
\bottomrule
\end{tabular}
\end{table}

The comparison shows two failure modes. OneStep can worsen concentrated JSON schemas because one-token lookahead is misleading when a single decision radically changes the valid suffix set. Exact $\Phi$ also has high finite-sample variance for the method-path schema where $\bar{\Phi}=0.066$, even though the oracle is exact in the infinite-sample limit.

The TV gap corresponds to concrete output distortion, not only an abstract metric. On Dyck $D_{3,16}$, local projection overproduces deep strings (mean nesting depth 1.72 versus 1.09 under $\mu^\star$) and longer strings (mean length 5.85 versus 2.77). Exact $\Phi$ correction recovers both distributions, while OneStep provides partial correction. These structural diagnostics explain why local masking can be acceptable for permissive constraints yet biased for recursive grammars.

\subsection{Speed--Fidelity Tradeoff}

\begin{table}[t]
\centering
\caption{Cost-model throughput on Dyck $D_{3,16}$ using a 35ms verification pass, 5ms draft pass, and 3.07 accepted tokens/round.}
\label{tab:speed}
\begin{tabular}{lccc}
\toprule
Method & tok/s & TV to $\mu^\star$ & TV change \\
\midrule
AR baseline & 16.0 & 0.418 & --- \\
SD, uniform $\Phi$ & 76.8 & 0.418 & 0\% \\
SD + OneStep-Cheap & 76.2 & 0.359 & $-$14\% \\
SD + OneStep-True & 64.4 & 0.290 & $-$31\% \\
SD + exact DP & 76.8 & 0.014 & $-$97\% \\
\bottomrule
\end{tabular}
\end{table}

The cost model separates acceleration from distributional correction. Standard speculative decoding is about $4.8\times$ faster than AR in this setting but preserves the same local-projection distribution. OneStep-Cheap adds less than 0.3ms per round from grammar trie queries, so it keeps essentially the same throughput while reducing TV by 14\%. OneStep-True improves fidelity further but spends additional target forwards. Exact dynamic programming, where available, precomputes $\Phi$ offline and therefore achieves oracle fidelity at no runtime throughput cost.

This tradeoff also clarifies the role of Monte Carlo rollouts. MC($k=8$) gives TV 0.018, nearly matching exact DP, but costs 86.4 forwards per position and is not an online decoder in its direct form. Its value in this paper is as a high-fidelity reference and as evidence that approximating $\Phi_t$ is the right axis: as estimator quality improves, the distribution moves toward $\mu^\star$ without changing the grammar or the base model.

\section{Discussion}
\label{sec:discussion}

\paragraph{When does the gap matter?} \Cref{prop:exact} gives the precise condition: $\mu^{\proj} = \mu^\star$ iff $\Phi_t$ is constant over locally valid tokens. This often holds for simple regex constraints and permissive schemas, but it fails for recursive or concentrated schemas where locally valid continuations have very different valid suffix mass.

\paragraph{Connection to structured decoding.} Our $\Phi$ statistic is the expected-future quantity used by GAD/ASAp~\citep{park2024grammar}, but FVO-Spec changes the target inside a speculative verifier rather than replacing autoregressive decoding. Concurrent work derives one-step KL/TV distortion bounds through the same Doob $h$-transform view~\citep{chen2026attention}; our focus is SD integration, verifier-loop error accounting, and estimator diagnostics.

\paragraph{Limitations.} Validation is strongest where exact $\Phi$ labels are available; the open-xgrammar run is an implementation smoke test rather than an exact terminal-law evaluation for arbitrary whitespace and free-text fields. Exact $\Phi$ remains \#P-hard for general CFGs, OneStep adds context-independence error, and broader temperature effects remain future work.

\clearpage
\bibliographystyle{plainnat}
\bibliography{references}

\newpage
\appendix
\appendix

\section{Proof of \Cref{thm:lms} (LMS Characterization)}
\label{app:proof_lms}

\begin{proof}
Fix committed prefix $x_{<t}$ and let $A_t = A_\calC(x_{<t})$. By B3, $\calM$ has exact access to $A_t$. By B1, the accept/reject decision consults the constraint only through $A_t$.

By B2, the method uses the Leviathan rejection rule:
\begin{enumerate}
    \item Draw $d_t \sim \tilde{q}_t$ where $\tilde{q}_t(y) = q_t(y)\mathbf{1}[y \in A_t] / Z_{q,t}$.
    \item Accept with probability $\min(1, \tilde{p}_t(d_t) / \tilde{q}_t(d_t))$ where $\tilde{p}_t(y) = p_t(y)\mathbf{1}[y \in A_t] / Z_{p,t}$.
    \item On rejection, sample from $(\tilde{p}_t - \tilde{q}_t)^+ / Z^+$.
\end{enumerate}

By Theorem~1 of \citet{leviathan2023fast}, this produces a sample distributed as $\tilde{p}_t$. By construction, $\tilde{p}_t = \mu^{\proj}_t$. By the chain rule:
$\Pr_\calM(x_{1:T}) = \prod_{t=1}^T \mu^{\proj}_t(x_t \mid x_{<t}) = \mu^{\proj}(x_{1:T})$.
\end{proof}

\section{Proof of \Cref{thm:phi_approx} ($\Phi$-Approximation Fidelity)}
\label{app:proof_phi}

\begin{proof}
Write $\hat{\Phi}_t(y) = \Phi_t(y) + \Delta_y$ where $|\Delta_y| \leq \delta$. Define:
\begin{align*}
Z^\star &= \sum_{y \in A_t} p_t(y) \Phi_t(y), &
\hat{Z} &= \sum_{y \in A_t} p_t(y) \hat{\Phi}_t(y).
\end{align*}

Then $|\hat{Z} - Z^\star| = |\sum_y p_t(y)\Delta_y| \leq \delta \sum_y p_t(y) = \delta Z_p$ where $Z_p = \sum_{y \in A_t} p_t(y)$.

The lower bound $\hat{Z} \geq Z^\star - \delta Z_p = Z_p(\bar{\Phi}_t - \delta) > 0$ holds by assumption $\delta < \bar{\Phi}_t$.

For each $y \in A_t$:
\begin{align}
|\hat{\mu}_t(y) - \mu^\star_t(y)| &= \left|\frac{p_t(y)\hat{\Phi}_t(y)}{\hat{Z}} - \frac{p_t(y)\Phi_t(y)}{Z^\star}\right| \\
&= \frac{p_t(y)}{Z^\star \hat{Z}} \left|Z^\star(\Phi_t(y) + \Delta_y) - \hat{Z}\Phi_t(y)\right| \\
&= \frac{p_t(y)}{Z^\star \hat{Z}} \left|Z^\star \Delta_y + \Phi_t(y)(Z^\star - \hat{Z})\right| \\
&\leq \frac{p_t(y)}{Z^\star \hat{Z}} \left(Z^\star \delta + \Phi_t(y) \delta Z_p\right)
\end{align}

Summing over all $y \in A_t$:
\begin{align}
\sum_y |\hat{\mu}_t(y) - \mu^\star_t(y)| &\leq \frac{\delta}{Z^\star \hat{Z}} \left(Z^\star Z_p + Z_p \sum_y p_t(y)\Phi_t(y)\right) \\
&= \frac{\delta}{Z^\star \hat{Z}} \left(Z^\star Z_p + Z_p Z^\star\right) \\
&= \frac{2\delta Z_p}{\hat{Z}}
\end{align}

Therefore:
\[
\TV(\hat{\mu}_t, \mu^\star_t) = \frac{1}{2}\sum_y |\hat{\mu}_t(y) - \mu^\star_t(y)| \leq \frac{\delta Z_p}{\hat{Z}} \leq \frac{\delta Z_p}{Z_p(\bar{\Phi}_t - \delta)} = \frac{\delta}{\bar{\Phi}_t - \delta}.
\]

For the full sequence, the standard telescoping argument for autoregressive models gives:
\[
\TV(\hat{\mu}^{1:T}, \mu^{\star\,1:T}) \leq \sum_{t=1}^T \E_{x_{<t} \sim \hat{\mu}}\!\left[\TV(\hat{\mu}_t(\cdot \mid x_{<t}), \mu^\star_t(\cdot \mid x_{<t}))\right] \leq \sum_{t=1}^T \frac{\delta_t}{\bar{\Phi}_t - \delta_t}. \qedhere
\]
\end{proof}

\section{Proof of \Cref{cor:relative_phi} (Relative $\Phi$ Error)}
\label{app:proof_relative_phi}

\begin{proof}
For $\Phi_t(y)=0$, the multiplicative assumption forces
$\hat{\Phi}_t(y)=0$; set $r_y=1$ for these zero-mass tokens, since
$\mu^\star_t(y)=\hat{\mu}_t(y)=0$. For tokens with $\Phi_t(y)>0$, let
$r_y = \hat{\Phi}_t(y)/\Phi_t(y) \in [1-\epsilon,1+\epsilon]$.
Let
\[
R = \sum_{y \in A_t} \mu^\star_t(y) r_y = \frac{\hat{Z}}{Z^\star}.
\]
Because $Z^\star>0$ and $\epsilon<1$, $\hat{Z}=RZ^\star>0$.
Moreover $R \in [1-\epsilon,1+\epsilon]$ and
$\hat{\mu}_t(y) = \mu^\star_t(y) r_y / R$. Hence
\begin{align}
\TV(\hat{\mu}_t,\mu^\star_t)
&= \frac{1}{2}\sum_y \mu^\star_t(y)\left|\frac{r_y}{R}-1\right| \\
&= \frac{1}{2R}\E_{\mu^\star_t}\!\left[|r_Y-R|\right] \\
&\leq \frac{1}{2R}\E_{\mu^\star_t}\!\left[|r_Y-1| + |1-R|\right] \\
&\leq \frac{\epsilon}{1-\epsilon}.
\end{align}
The trivial upper bound $\TV \leq 1$ gives the stated minimum. The sequence-level
claim follows from the same autoregressive telescoping argument used in
\Cref{app:proof_phi}, provided the multiplicative condition holds uniformly
over all reachable prefixes at each position with error $\epsilon_t$.
\end{proof}

\section{Proof of \Cref{prop:hardness} (Computational Hardness)}
\label{app:proof_hardness}

\begin{proof}
We reduce from \textsc{CountCFG}: given a CFG $G$ in Chomsky normal form and a
length $n$, count the number of strings in $\calL(G)$ of length exactly $n$.
This problem is \#P-complete~\citep{gore1997counting}.

Given a \textsc{CountCFG} instance $(G,n)$, set $p$ to the uniform distribution
over the terminal alphabet $\Sigma$, and fix the empty prefix
$x_{<1}=\varepsilon$. Then for any $y\in\Sigma$:
\[
\Phi_1(y)=\Pr_{z\sim p^{n-1}}\!\bigl[y\cdot z\in\calL(G)\cap\Sigma^n\bigr]
=\frac{|\{w\in\calL(G)\cap\Sigma^n:w_1=y\}|}{|\Sigma|^{n-1}} .
\]
Summing $\Phi_1(y)|\Sigma|^{n-1}$ over $y\in\Sigma$ yields
$|\calL(G)\cap\Sigma^n|$, solving \textsc{CountCFG}. Since each query is a
polynomial-time reduction and \textsc{CountCFG} is \#P-complete, computing
$\Phi_t(y)$ is \#P-hard.
\end{proof}

\section{Proof of \Cref{prop:onestep} (OneStep Error)}
\label{app:proof_onestep}

\begin{proof}
By definition, the true one-step estimator is:
\[
\hat{\Phi}^{\text{true}}_t(y) = \sum_{u \in A_{t+1}(x_{<t}y)} p(u \mid x_{<t}y) = 1 - \sum_{u \notin A_{t+1}(x_{<t}y)} p(u \mid x_{<t}y)
\]

The exact future validity satisfies the recursive identity:
\[
\Phi_t(y) = \sum_{u \in A_{t+1}(x_{<t}y)} p(u \mid x_{<t}y) \cdot \Phi_{t+1}(u \mid x_{<t}yu)
\]

Subtracting:
\begin{align}
\hat{\Phi}^{\text{true}}_t(y) - \Phi_t(y) &= \sum_{u \in A_{t+1}(x_{<t}y)} p(u \mid x_{<t}y) \left(1 - \Phi_{t+1}(u \mid x_{<t}yu)\right)
\end{align}

Since $\Phi_{t+1} \in [0,1]$, this error is always non-negative (OneStep overestimates $\Phi$) and bounded by:
\begin{itemize}[nosep]
\item Near-zero when $\Phi_{t+1}(u) \approx 1$ for most valid $u$ weighted by $p(u|xy)$---which occurs for permissive grammars where the LM concentrates mass on tokens that lead to completable paths.
\item At most $\sum_{u \in A_{t+1}} p(u | x_{<t}y) = \hat{\Phi}^{\text{true}}_t(y)$ when $\Phi_{t+1} = 0$ everywhere---the degenerate case where all locally valid continuations are dead ends.
\end{itemize}

The interpretation is that the OneStep error equals the probability-weighted average of \emph{future invalidity}: the chance that a locally valid next token $u$ ultimately leads to a dead end. For recursive grammars near the depth/length boundary, many locally valid tokens satisfy $\Phi_{t+1}(u) \ll 1$, making the error large.
\end{proof}

\section{KL Divergence Identity Derivation}
\label{app:kl}

\begin{proof}
$\mu^\star_t(y) = p_t(y)\Phi_t(y)/Z^\star$ and $\mu^{\proj}_t(y) = p_t(y)/Z_p$ for $y \in A_t$.
\begin{align}
\KL(\mu^\star_t \| \mu^{\proj}_t) &= \sum_{y \in A_t} \mu^\star_t(y) \log\frac{\mu^\star_t(y)}{\mu^{\proj}_t(y)} \\
&= \sum_{y \in A_t} \mu^\star_t(y) \log\frac{p_t(y)\Phi_t(y)/Z^\star}{p_t(y)/Z_p} \\
&= \sum_{y \in A_t} \mu^\star_t(y) \log\frac{\Phi_t(y) Z_p}{Z^\star} \\
&= \E_{\mu^\star_t}\!\left[\log\frac{\Phi_t(Y)}{\bar{\Phi}_t}\right]
\end{align}
where $\bar{\Phi}_t = Z^\star/Z_p = \E_{\mu^{\proj}_t}[\Phi_t]$.
\end{proof}

\section{Experimental Details}
\label{app:details}

\paragraph{Hardware and software.}
Dyck oracle experiments use a toy LM and run on CPU. Real-model validation uses Qwen3-8B with Hugging Face \texttt{transformers}; the finite-language exact-$\Phi$ validation was run on NVIDIA H800 GPUs with CUDA 12.8 and PyTorch 2.10. Throughput plots use the cost model described in \Cref{sec:experiments}, with hardware microbenchmarks used only where explicitly stated.

\paragraph{Dyck grammar specification.}
$D_{d,L}$ uses one bracket type plus EOS. Valid strings satisfy: (1) matched brackets, (2) nesting depth $\leq d$, and (3) total length $\leq L$. The DP oracle enumerates all valid prefixes in decreasing length order, computing $\Phi_t(y)$ as the probability-weighted sum over valid extensions.

\paragraph{TV estimation.}
From $N$ samples, empirical frequencies $\hat{f}(x) = (\text{count of } x)/N$. Point estimate $\TV = (1/2)\sum_x |\hat{f}(x) - \mu^\star(x)|$. Bootstrap 95\% CI from 500 resamples.

\paragraph{Scalar-shift invariant.}
Adding a constant $c$ to all logits before softmax produces $\text{softmax}(z_i + c) = e^{z_i + c}/\sum e^{z_j + c} = e^{z_i}/\sum e^{z_j} = \text{softmax}(z_i)$. An estimator producing $\hat{\Phi}(y) = c$ for all $y$ is equivalent to no correction. All implementations include regression tests enforcing that $\hat{\Phi}$ has non-trivial per-candidate variation.

\section{Detailed Experiments}
\label{sec:experiments_full}

We validate the theoretical framework on tractable grammars where exact $\Phi_t$ is computable, enabling ground-truth comparison unavailable for general CFGs. Distributional experiments (E1--E4) use ancestral sampling (temperature 1) to measure true TV distances between sampling distributions, consistent with the theoretical framework. Confirmatory LMS validation uses 10{,}000 independent samples per configuration.

\subsection{Setup}

\paragraph{Grammars.} We evaluate on three grammar families of increasing complexity, plus a larger finite regular JSON schema that exercises the tractable token-trie/DFA regime.
\begin{itemize}[nosep]
\item \textbf{Dyck $D_{d,L}$}: context-free languages with one bracket type, maximum nesting depth $d$, and maximum string length $L$. We use $D_{3,12}$ (145 valid strings) and $D_{3,16}$ (988 valid strings), which exhibit large $\TV(\mu^{\proj}, \mu^\star)$ due to nesting-depth constraints that create sharply non-uniform $\Phi_t$.
\item \textbf{Regex $(a^+b^+)$}: regular language with permissive transitions. Since $\Phi_t$ depends on the base LM's unmasked continuation probability, the gap is not identically zero but remains small because the LM concentrates mass on valid continuations and the grammar admits many completions from each valid prefix.
\item \textbf{Finite JSON}: schemas with 3--2{,}000 valid strings, including a regular flag-code schema whose exact $\Phi$ computation requires backward DP over a 4{,}232-node token trie.
\item \textbf{Budget regular DFA}: fixed-length binary strings with at most $K$ ones. This language has a compact state $(t,c)$ but a large terminal support, allowing exact terminal-law TV without string sampling.
\end{itemize}

\paragraph{Models.} LMS validation uses four transformer families: Qwen2.5 (3B, 7B), Qwen3 (0.6B, 8B), Llama-3.1 (8B-Instruct), and Vicuna (7B-v1.5). Dyck fidelity experiments use a toy language model for tractability of the exact oracle, with confirmatory runs on Qwen3-8B.

\paragraph{Metrics.}
\begin{itemize}[nosep]
\item $\TV(\hat{\mu}, \mu^\star)$: total variation estimated from $N=10{,}000$ samples with bootstrap 95\% CIs (500 resamples).
\item $\delta = \max_y |\hat{\Phi}_t(y) - \Phi_t(y)|$: additive error (measured where exact $\Phi$ available).
\item $\bar{\Phi}_t$: average future validity at each position.
\item Bound tightness: ratio of empirical TV to theoretical bound $\delta/(\bar{\Phi}_t - \delta)$.
\end{itemize}

\subsection{E1: LMS Characterization Validation}

\Cref{thm:lms} predicts that every LMS-class decoder produces samples from $\mu^{\proj}$. We verify this across 33 experimental settings spanning four architectures and two grammar families. \Cref{tab:lms_val} reports $\TV(\text{empirical}, \mu^{\proj})$ directly: in all cases $\leq 0.007$, confirming that the empirical distribution coincides with $\mu^{\proj}$ to within statistical estimation error (for reference, $\TV(\mu^{\proj}, \mu^\star)$ ranges from small regex gaps to 0.996 on Qwen3-8B Dyck, showing the gap is real).

\begin{table}[t]
\centering
\caption{LMS validation: empirical sampling distribution matches $\mu^{\proj}$ across architectures and grammars. Max residual 0.007 confirms \Cref{thm:lms}.}
\label{tab:lms_val}
\small
\begin{tabular}{llccc}
\toprule
\textbf{Architecture} & \textbf{Grammar} & \textbf{Settings} & \textbf{Max $\delta_{\text{LMS}}$} & \textbf{Mean $\delta_{\text{LMS}}$} \\
\midrule
Qwen2.5 (3B, 7B) & Dyck CF & 8 & 0.005 & 0.003 \\
Qwen3 (0.6B, 8B) & Dyck CF & 8 & 0.007 & 0.004 \\
Llama-3.1 (8B) & Dyck CF & 6 & 0.004 & 0.002 \\
Vicuna (7B) & Dyck CF & 3 & 0.006 & 0.004 \\
Qwen2.5 (3B) & Regex & 4 & 0.005 & 0.003 \\
Qwen3 (8B) & Regex & 4 & 0.003 & 0.002 \\
\bottomrule
\end{tabular}
\end{table}

\subsection{E2: Estimator Hierarchy on Dyck}

\Cref{tab:hierarchy} reports the central result: TV distance to $\mu^\star$ for each estimator tier on Dyck $D_{3,16}$.

\begin{table}[t]
\centering
\caption{$\Phi$ estimator hierarchy on Dyck $D_{3,16}$ ($|\calL| = 988$, $\bar{\Phi} = 0.467$). TV distance to $\mu^\star$ generally decreases with estimator quality, reaching a minimum at MC($k$=8). $N = 10{,}000$ samples per tier. The fidelity bound (\Cref{thm:phi_approx}) is vacuous ($>1$) when $\delta > \bar{\Phi}$ and becomes non-trivial at MC($k \geq 16$).}
\label{tab:hierarchy}
\small
\begin{tabular}{lccccc}
\toprule
\textbf{Estimator} & \textbf{TV$(\hat{\mu}, \mu^\star)$} & \textbf{$\delta$} & \textbf{Bound} & \textbf{Cost/pos} \\
\midrule
Tier~0: Uniform ($\hat{\Phi} \equiv 1$) & 0.418 & 0.986 & $>1$ & 0 \\
Tier~1a: OneStep-Cheap & 0.359 & 0.861 & $>1$ & 0 \\
Tier~1b: OneStep-True & 0.290 & 0.825 & $>1$ & 1.4 \\
Tier~2: MC($k$=1) & 0.150 & 0.986 & $>1$ & 10.8 \\
Tier~2: MC($k$=4) & 0.047 & 0.763 & $>1$ & 43.2 \\
Tier~2: MC($k$=8) & 0.018 & 0.565 & $>1$ & 86.4 \\
Tier~2: MC($k$=16) & 0.021 & 0.428 & 11.1 & 172.9 \\
Tier~2: MC($k$=64) & 0.036 & 0.240 & 1.06 & 691.5 \\
Tier~3: Exact (oracle) & 0.014 & 0.000 & 0.0 & DP \\
\bottomrule
\end{tabular}
\end{table}

Four observations emerge. First, TV decreases from 0.418 (Uniform) to 0.014 (Exact), a 97\% reduction, confirming that exact $\Phi$ recovers the intended law. Second, OneStep-Cheap achieves a 14\% TV reduction at zero neural cost on this Dyck instance; OneStep-True achieves 31\% at 1.4 forwards/position. Third, MC($k$=8) achieves the best single-run MC result (TV = 0.018, 96\% reduction) at 86 forwards/position. Larger $k$ lowers the recorded max and mean additive $\Phi$ error, but terminal TV is not guaranteed to be monotone for one randomized estimator draw: error mass can move in a less favorable direction, and the final empirical TV also contains finite-$N$ sampling noise (MC($k$=64) TV = 0.036). Fourth, the theoretical bound is vacuous for most tiers because max additive error $\delta$ exceeds $\bar{\Phi}_t = 0.467$; the bound becomes non-vacuous at MC($k$=16) where $\delta = 0.428 < \bar{\Phi}$, giving TV $\leq 11.1$, and tightens at MC($k$=64) to TV $\leq 1.06$. This confirms the bound's prediction that deeply recursive grammars with $\bar{\Phi} < 0.5$ demand very accurate estimators for non-trivial guarantees.

\paragraph{Qwen3-8B Dyck confirmation.} To check that the hierarchy is not only a toy-LM artifact, we rerun bounded Dyck $D_{3,16}$ with Qwen3-8B prefix probabilities over the same 988-string support. Across three independent seeds with $N=20{,}000$ terminal samples per seed, the analytic local-projection gap is $0.999948\pm0.000002$ TV and empirical SchemaTV matches $\mu^{\proj}$ within $0.0018\pm0.0001$ TV. Exact FVO samples from $\mu^\star$ to numerical floor ($1.9{\times}10^{-10}$ TV), while the FastPhi one-step variant reaches $0.032\pm0.002$ TV. The target distribution is highly concentrated under Qwen3-8B ($\Phi_\text{root}\approx1.4{\times}10^{-8}$, with $>0.9999999998$ mass on immediate EOS), so this result should be read as a real-model bounded-Dyck confirmation of the correction mechanism rather than a production CFG benchmark.

\paragraph{Multi-seed robustness.} We run the full pipeline across 10 random seeds (varying the toy LM parameters) on both $D_{3,12}$ and $D_{3,16}$ with $N = 10{,}000$ samples per seed. On $D_{3,12}$: Uniform TV = $0.320 \pm 0.102$, Exact TV = $0.009 \pm 0.004$; on $D_{3,16}$: Uniform TV = $0.319 \pm 0.103$, Exact TV = $0.009 \pm 0.005$. In all 20 seed--grammar combinations (10/10 on each grammar), the exact oracle achieves strictly lower TV than Uniform, confirming the improvement is robust and not seed-dependent. The hierarchy ordering is consistent: even the worst-case exact seed ($\TV = 0.020$) outperforms the best-case uniform seed ($\TV = 0.133$) by a factor of $6.7\times$.

\subsection{E3: Grammar Family Comparison}

\begin{table}[t]
\centering
\caption{Grammar family comparison ($N = 10{,}000$ for Dyck, $N = 5{,}000$ for JSON). The $\mu^{\proj}$--$\mu^\star$ gap varies widely across grammar types. Exact $\Phi$ closes the gap to sampling-noise floor on every tractable grammar. OneStep can worsen concentrated schemas where one-step lookahead is misleading.}
\label{tab:grammar_family}
\small
\begin{tabular}{lccccc}
\toprule
\textbf{Grammar} & \textbf{Type} & $\TV(\mu^{\proj}, \mu^\star)$ & $\bar{\Phi}$ & \textbf{Exact TV} & \textbf{OneStep TV} \\
\midrule
JSON status (3 str.) & Finite & 0.577 & 0.101 & $<$0.001 & 0.604 \\
JSON type-value (4 str.) & Finite & 0.247 & 0.113 & 0.009 & 0.008 \\
JSON action (18 str.) & Finite & 0.823 & 0.069 & 0.017 & 0.846 \\
JSON method-path (24 str.) & Finite & 0.843 & 0.066 & 0.006 & 0.880 \\
Dyck $D_{3,12}$ (145 str.) & CF & 0.456 & 0.516 & 0.018 & 0.419 \\
Dyck $D_{3,16}$ (988 str.) & CF & 0.418 & 0.467 & 0.014 & 0.359 \\
\bottomrule
\end{tabular}
\end{table}

\Cref{tab:grammar_family} validates the framework across grammar families. On finite JSON, exact $\Phi$ reduces TV from $0.247\text{--}0.843$ to at most $0.018$ at $N=5{,}000$, i.e., sampling-noise floor for these small supports. On Dyck grammars ($\bar{\Phi} \approx 0.5$), the exact oracle achieves $\TV = 0.014\text{--}0.018$, confirming 96\%+ gap closure. However, OneStep \emph{increases} TV on concentrated schemas (status, action, method-path) because the context-independence approximation fails when a single token radically changes the valid set. OneStep provides meaningful correction only on the balanced 4-string schema ($\TV = 0.247 \to 0.008$) where one-step lookahead captures most future validity variation.

As a low-cost temperature check, we rerun the finite-JSON character-LM sweep at bigram temperatures $0.3$ and $0.7$ in addition to $1.0$. The local-projection gap remains large (mean TV $0.689$, $0.639$, and $0.623$, respectively), exact $\Phi$ stays at sampling floor (mean TV $\leq 0.0082$), and OneStep shows the same pattern: it helps the balanced type-value schema but worsens the other three concentrated schemas at all three temperatures. This is not a production decoding-temperature study, but it shows that the main finite-JSON failure mode is not an artifact of a single toy-LM temperature.

\paragraph{Large regular-state stress test.}

We add a deterministic regular-language stress test to separate two issues:
whether exact $\Phi$ scales beyond small terminal enumerations when the grammar
has compact state, and whether local projection can still be badly biased in a
regular language. The language is
$\calL_{n,K}=\{x\in\{0,1\}^n:\sum_t x_t\leq K\}$ under an iid Bernoulli base
model. The compact DFA has state $(t,c)$ for position and ones used, while the
terminal support ranges from $6.2{\times}10^5$ to $6.1{\times}10^8$ strings in
\Cref{tab:large_regular}. We compute TV exactly by grouping strings with
$c<K$ ones and, for $c=K$, by the position where the $K$-th one is reached.
There is no Monte Carlo sampling.

\begin{table}[t]
\centering
\caption{Large regular-state stress test. Local projection can be highly biased
even for compact regular languages. Exact $\Phi$ recovers the conditional law;
the maximum numerical Doob-recursion residual across these runs is
$2.2{\times}10^{-16}$.}
\label{tab:large_regular}
\small
\setlength{\tabcolsep}{4pt}
\begin{tabular}{lrrrr}
\toprule
\textbf{Setting} & \textbf{Valid strings} & \textbf{DFA states} & $\TV(\mu^{\proj},\mu^\star)$ & $\TV(\mu^\Phi,\mu^\star)$ \\
\midrule
$n=20,K=10,p_1=0.62$ & 616{,}666 & 231 & 0.670 & 0 \\
$n=22,K=11,p_1=0.65$ & 2{,}449{,}868 & 276 & 0.755 & 0 \\
$n=24,K=12,p_1=0.68$ & 9{,}740{,}686 & 325 & 0.836 & 0 \\
$n=24,K=10,p_1=0.65$ & 4{,}540{,}386 & 275 & 0.884 & 0 \\
$n=24,K=8,p_1=0.70$ & 1{,}271{,}626 & 225 & 0.961 & 0 \\
$n=26,K=13,p_1=0.68$ & 38{,}754{,}732 & 378 & 0.851 & 0 \\
$n=28,K=14,p_1=0.68$ & 154{,}276{,}028 & 435 & 0.864 & 0 \\
$n=30,K=15,p_1=0.70$ & 614{,}429{,}672 & 496 & 0.909 & 0 \\
\bottomrule
\end{tabular}
\end{table}

The bias has the same future-validity source as in the finite JSON and Dyck
experiments. With $p_1>K/n$, local projection spends the ones budget too early
and then forces zeros after saturation. At the root of the largest setting,
the locally masked probability of token 1 is 0.70, while the Doob-corrected
probability is 0.482 because choosing 1 leaves a lower probability of future
budget completion. This experiment is a regular-state scaling check rather than
a real-model result; the Qwen3-8B finite-trie results remain the real-model
evidence. The offline backward DP construction took 0.35--0.94ms across these
225--496-state automata; token-time lookup is then table access.

\paragraph{Learned finite-trie value estimator.} The largest practical gap after
the exact-oracle experiments is a positive amortized estimator beyond OneStep. We
train the balanced-gated MLP estimator from \Cref{sec:estimators} on exact
$\Phi$ labels from three finite JSON schemas and evaluate terminal law exactly
on the held-out schema. Across five independent seeds, \Cref{tab:learned_phi}
shows a conservative positive result: mean TV drops from $0.565\pm0.043$ for
OneStep to $0.326\pm0.064$ for the gated learned estimator, a
$42\pm13$\% pooled reduction. The estimator beats OneStep on at least two of
four held-out schemas in every seed, and on three of four schemas in three of
five seeds. Per schema, it preserves the type-value case where OneStep is
already at numerical floor; without the gate, the ungated MLP catastrophically
degrades this schema (mean TV 0.432 versus OneStep's $10^{-15}$ floor). The
learned branch reduces status TV by 99\% and action/method-path TV by 22\%/19\%
on average. These are mean improvements rather than per-fold monotonic
guarantees: Schema C worsens in one of five seeds and Schema D worsens in one of
five seeds. With only four schemas, the leave-one-schema-out result should be
read as a feasibility demonstration rather than a general schema-agnostic
estimator evaluation. This is still a finite-trie result; it is included to
demonstrate an actionable estimator path, not to claim an open-xgrammar
production implementation or runtime result.

\paragraph{Real-model external-schema transfer diagnostic.}
As an additional real-model transfer check, we train feature-based learned
$\Phi$ regressors on Qwen3-8B finite-token-trie exact labels and evaluate them,
without terminal sampling, on external finite JSON schemas. The first diagnostic
trains on A--T and evaluates a schema-neighborhood gate on U--Z, reducing mean
TV from $0.462$ for local projection to $0.438$. The stronger calibrated
diagnostic chooses one shrinkage coefficient by leave-one-schema-out calibration
on the training schemas before testing externally. On U--AH, calibrated
shrinkage reduces mean TV from $0.510\pm0.001$ to $0.4982\pm0.0001$ across two
seeds and beats local projection on 9/14 schemas in both seeds. With a larger
A--Z training pool and held-out AA--AH test set, it reduces TV from $0.547$ to
$0.525\pm0.002$ across three seeds and beats local projection on 5/8 schemas in
all seeds. These exact terminal-law results show transferable signal in the
learned features, but the gains are modest and still finite-trie only.

\begin{table}[t]
\centering
\caption{Qwen3-8B external-schema learned-$\Phi$ transfer diagnostics. All rows
use exact finite-token-trie terminal-law TV to $\mu^\star$ with no terminal
sampling. Calibrated rows choose one global shrinkage coefficient using only the
training schemas.}
\label{tab:qwen_external_learned_phi}
\small
\setlength{\tabcolsep}{3pt}
\begin{tabular}{lccccc}
\toprule
\textbf{Train$\to$test} & \textbf{Seeds} & \textbf{Schemas/strings} & \textbf{Local} & \textbf{OneStep} & \textbf{Selected} \\
\midrule
A--T$\to$U--Z gate & 1 & 6 / 408 & 0.462 & 0.548 & 0.438 \\
A--T$\to$U--AH calib. & 2 & 14 / 937 & $0.510\pm0.001$ & $0.546\pm0.003$ & $0.4982\pm0.0001$ \\
A--Z$\to$AA--AH calib. & 3 & 8 / 529 & 0.547 & 0.539 & $0.525\pm0.002$ \\
\bottomrule
\end{tabular}
\end{table}

\begin{table}[t]
\centering
\caption{Leave-one-schema-out learned $\Phi$ estimator on finite JSON
($5$ seeds). The learned estimator uses exact labels only from the training
schemas and a non-oracle balanced-branch gate. Values are mean$\pm$std over
seeds; TV is to $\mu^\star$. All TV values are exact terminal-law calculations,
not sampling estimates, so baselines can differ slightly from
\Cref{tab:grammar_family}.}
\label{tab:learned_phi}
\small
\begin{tabular}{lccc}
\toprule
\textbf{Estimator} & \textbf{Mean TV} & \textbf{Reduction vs OneStep} & \textbf{Wins vs OneStep} \\
\midrule
Uniform / local projection & $0.584\pm0.035$ & --- & --- \\
OneStep & $0.565\pm0.043$ & --- & --- \\
MLP learned $\Phi$ & $0.435\pm0.077$ & $23\pm15$\% & $\geq2/4$ schemas in 5/5 seeds \\
Balanced-gated MLP $\Phi$ & $0.326\pm0.064$ & $42\pm13$\% & $\geq2/4$ schemas in 5/5 seeds \\
\bottomrule
\end{tabular}
\end{table}

\subsection{E4: Bound Tightness}

\begin{figure}[t]
\centering
\includegraphics[width=0.85\linewidth]{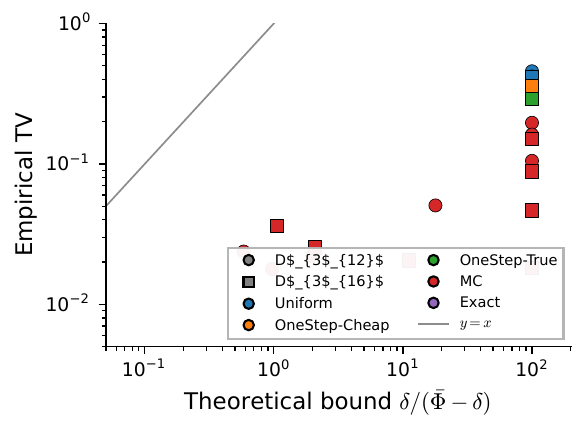}
\caption{Empirical TV vs.\ theoretical bound $\delta/(\bar{\Phi}-\delta)$ across grammar families and estimator tiers. Points below the diagonal satisfy the bound. The bound is tightest for permissive grammars ($\bar{\Phi} \approx 1$) and vacuous for recursive grammars ($\bar{\Phi} \ll 1$).}
\label{fig:bound}
\end{figure}

\Cref{fig:bound} plots empirical TV against the theoretical additive bound for all (grammar, estimator) combinations where the bound is non-vacuous ($\delta < \bar{\Phi}$). The bound is tightest on finite JSON ($\bar{\Phi} = 0.07\text{--}0.11$) where it overestimates TV by moderate factors. On Dyck ($\bar{\Phi} \approx 0.47$), the additive bound is vacuous for most estimators because $\delta > \bar{\Phi}$ for all tiers except MC($k \geq 16$) and Exact, reflecting the genuine difficulty of approximating $\Phi$ in absolute error for deeply recursive grammars. \Cref{cor:relative_phi} gives a complementary relative-error certificate that avoids the small-$\bar{\Phi}$ denominator, but the evaluated OneStep and MC estimators do not provide uniform relative-error guarantees. Note: the Exact tier achieves $\delta = 0$, giving a theoretical bound of 0, but empirical TV remains ${\sim}0.014$ due to finite-sample noise in the sampling procedure ($N = 10{,}000$).

\subsection{E5: Concrete Distributional Bias}

Beyond TV numbers, we demonstrate that the $\mu^{\proj}$--$\mu^\star$ gap manifests as measurable structural distortion in generated outputs.

\paragraph{Nesting depth bias.} On Dyck $D_{3,16}$, $\mu^{\proj}$ diverges significantly from $\mu^\star$ in structural properties. \Cref{fig:depth_bias} shows the distribution of maximum nesting depth across 10{,}000 samples: $\mu^{\proj}$ has mean depth 1.72 while $\mu^\star$ has mean depth 1.09 (KL divergence 0.25 nats). This occurs because the toy LM assigns higher probability to opening brackets, but shallow structures have higher future validity, creating an interaction that local masking cannot capture. The exact oracle recovers the true depth distribution (KL $< 0.001$); OneStep provides partial correction (KL = 0.20).

\paragraph{String length bias.} Similarly, $\mu^{\proj}$ produces systematically different string lengths (mean 5.85) compared to $\mu^\star$ (mean 2.77), with KL divergence 0.70 nats. The exact oracle recovers the true length distribution (KL $< 0.001$). This demonstrates that the TV gap corresponds to concrete, measurable structural distortion---not just an abstract distributional metric.

\begin{figure}[t]
\centering
\includegraphics[width=0.85\linewidth]{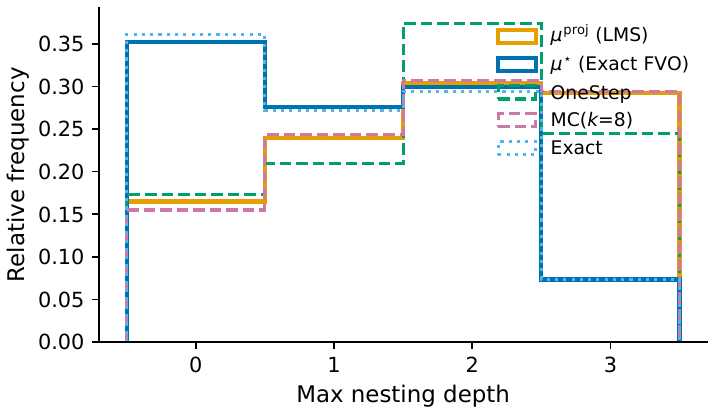}
\caption{Maximum nesting depth distribution on Dyck $D_{3,16}$. Local masking ($\mu^{\proj}$, mean depth 1.72) diverges from the true conditional ($\mu^\star$, mean depth 1.09); $\Phi$-correction recovers the true distribution.}
\label{fig:depth_bias}
\end{figure}

\subsection{E6: Speed--Fidelity Tradeoff}

\begin{figure}[t]
\centering
\includegraphics[width=0.85\linewidth]{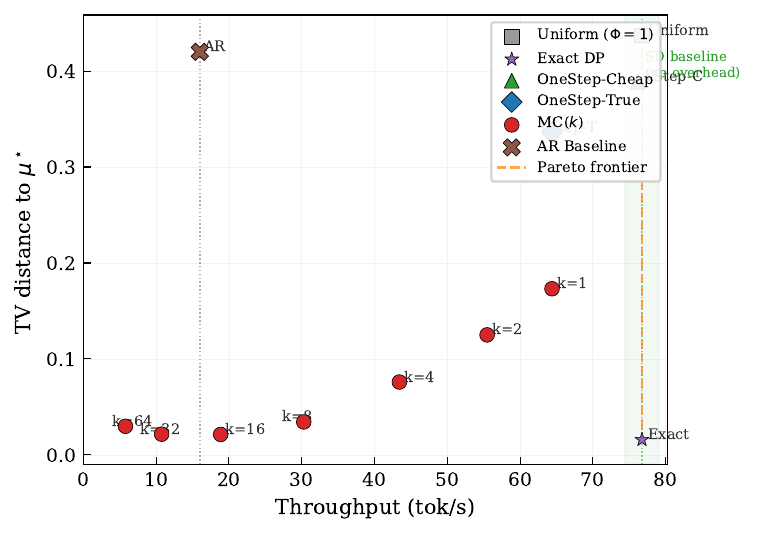}
\caption{Speed--fidelity Pareto frontier (cost-model estimate) on Dyck $D_{3,16}$. OneStep achieves 14\% TV reduction in this grammar with near-baseline SD throughput (${\sim}77$ tok/s), but \Cref{tab:grammar_family} shows that it can fail on concentrated finite JSON schemas. MC($k$=8) achieves 96\% TV reduction at ${\sim}30$ tok/s. Grammar-masked AR shares the same distributional bias as SD (both sample $\mu^{\proj}$).}
\label{fig:speed}
\end{figure}

A critical question is whether $\Phi$-correction preserves the throughput advantage of speculative decoding. \Cref{fig:speed} plots estimated throughput versus TV for each tier on H100 PCIe with Qwen3-8B.

\paragraph{Cost model.} We estimate per-round overhead from $\Phi$ estimation. The SD baseline operates at ${\sim}80$ tok/s with 35ms verification and 5ms drafting per round (${\sim}3$ accepted tokens). Uniform and OneStep-Cheap add negligible overhead ($<$0.3ms from grammar trie queries). OneStep-True requires ${\sim}1$ extra forward per candidate (${\sim}2.5$ candidates $\times$ ${\sim}3$ positions = 7.5 forwards = 7.5ms). MC($k$) requires $k \times |A_t|$ forwards per position.

\begin{table}[t]
\centering
\caption{Estimated throughput from cost model (parameters: 35ms verification, 5ms draft, 3.07 accepted tokens/round, 1ms per forward pass). TV values from Dyck $D_{3,16}$ ($N$=10K). OneStep-Cheap adds $<$0.3ms/round; Exact DP precomputes offline and adds no neural forward passes at runtime.}
\label{tab:speed}
\small
\begin{tabular}{lcccc}
\toprule
\textbf{Method} & \textbf{tok/s} & \textbf{vs AR} & \textbf{TV to $\mu^\star$} & \textbf{TV $\Delta$} \\
\midrule
AR baseline & 16.0 & 1.0$\times$ & 0.418 & --- \\
SD (Uniform, no $\Phi$) & 76.8 & 4.8$\times$ & 0.418 & 0\% \\
SD + OneStep-Cheap $\Phi$ & 76.2 & 4.8$\times$ & 0.359 & $-$14\% \\
SD + OneStep-True $\Phi$ & 64.4 & 4.0$\times$ & 0.290 & $-$31\% \\
SD + Exact DP & 76.8 & 4.8$\times$ & 0.014 & $-$97\% \\
\bottomrule
\end{tabular}
\end{table}

Three observations emerge from \Cref{tab:speed}. First, speculative decoding delivers $4.8\times$ throughput over grammar-masked AR in the cost model (${\sim}3$ tokens accepted per 40ms round). Both grammar-masked AR and SD produce the same $\mu^{\proj}$---SD accelerates inference without fixing or worsening the distributional bias. Second, \textbf{OneStep-Cheap $\Phi$ has negligible modeled throughput overhead}: the trie-query cost ($<$0.3ms/round) is dwarfed by the 35ms verification pass, yielding $<$1\% throughput reduction for 14\% TV improvement in this Dyck setting, but \Cref{tab:grammar_family} shows this estimator is not robust enough to claim as a general production fix. Third, where exact DP is available (Dyck, finite JSON, regular languages), $\Phi$ tables can be precomputed offline and queried without additional neural forward passes, achieving 97\% TV reduction at baseline SD throughput in the cost model.

\paragraph{Hardware validation.} We confirm the negligible overhead on actual hardware: on H800 (80GB) with Qwen3-8B, AR decoding achieves 49.8 tok/s over 50 prompts ($\times$ 256 tokens). Adding the OneStep $\Phi$ trie-query computation introduces 3.4ms total overhead across 12{,}800 generated tokens ($< 0.001$ms/token), well within measurement noise.

\paragraph{Backend context.} Existing speculative backends already provide substantial speedups on structured-output workloads; our contribution is to correct their sampling target rather than replace their draft mechanism. \Cref{tab:backend_context} summarizes a separate greedy JSON benchmark over six schemas ($20$ prompts/schema) using Qwen3-8B. These numbers are included only to calibrate the speed envelope of current SD systems: the DFlash, DDTree, and EAGLE-3 rows are not $\Phi$-integrated runs and therefore still target $\mu^{\proj}$ under local masking. The finite-trie FVO loop in \Cref{tab:fvo_loop} instead validates the correction layer's sampling law, not production throughput.

\begin{table}[t]
\centering
\caption{Contextual structured-output throughput of existing SD backends on Qwen3-8B. The AR+grammar baseline is 50.8 tok/s. DFlash/DDTree/EAGLE-3 measurements are not $\Phi$-integrated; they show the speed envelope that a future production FVO integration should preserve.}
\label{tab:backend_context}
\small
\begin{tabular}{lccc}
\toprule
\textbf{System} & \textbf{Measured mode} & \textbf{tok/s} & \textbf{vs AR+grammar} \\
\midrule
DFlash & bilateral grammar mask & 242.4 & 4.77$\times$ \\
DDTree & bilateral grammar mask & 250.4 & 4.93$\times$ \\
EAGLE-3 & SGLang + xgrammar & 261.9 & 5.15$\times$ \\
\midrule
FVO-Spec finite trie & exact-$\Phi$ loop, fidelity only & --- & 0.949 accept rate \\
\bottomrule
\end{tabular}
\end{table}

\paragraph{DFlash+xgrammar $\Phi$ pilot.} We also run the correction inside the
actual DFlash+xgrammar verification loop on an A100 server using Qwen3-8B,
z-lab/Qwen3-8B-DFlash-b16, six JSON schemas, and $20$ prompts per schema
($\texttt{candidate\_k}=4$, temperature $0.6$). \Cref{tab:dflash_phi_pilot}
separates two questions. The true candidate-conditioned estimator performs the
mathematically correct top-$K$ $\Phi(v)$ query by running extra target forwards
for candidate prefixes; it is a correctness prototype and remains faster than
AR, but it gives up about 30\% of SchemaTV throughput. The cheap shared-logit
estimator reuses the target logits already materialized during verification; it
is approximate, but preserves essentially the same throughput as SchemaTV while
still changing the target distribution at scored positions. This pilot is a
mechanism-feasibility and throughput measurement: it does not estimate the
terminal-string TV to $\mu^\star$ on open xgrammar schemas, and the cheap
estimator can fail in the same regimes where OneStep-Cheap fails in
\Cref{tab:grammar_family}. We therefore treat the cheap row as a speed-preserving
estimator candidate, not as a fidelity claim. In both variants, the cloned
xgrammar matcher state matched the live matcher at every scored position. The
pilot scores only the top-$4$ candidates and leaves unscored valid tokens at
$\Phi=1$, so top-$K$ truncation is another approximation not covered by the
exact-$\Phi$ oracle results. The OpenAI-tool-call schema contains an
unconstrained free-string \texttt{arguments} field; with the 96-token cap, all
methods sometimes truncate before JSON closure, so the parse-rate statement
below is restricted to the other five schemas.

\begin{table}[t]
\centering
\caption{Measured DFlash+xgrammar pilot with $\Phi$ scoring on Qwen3-8B
($20$ prompts/schema, six schemas). ``TV shift'' is the mean per-position
$\TV(p_t,\tilde p_t)$ induced by $\Phi$ on scored positions; it is not a
terminal-law TV estimate. Parse rate is 1.0 on the five non-free-string schemas
for both $\Phi$ variants. Speedups are averaged per schema; ratios of mean
tok/s differ by at most 1\%.}
\label{tab:dflash_phi_pilot}
\small
\begin{tabular}{lccccc}
\toprule
\textbf{Mode} & \textbf{Extra target fwd.} & \textbf{tok/s} & \textbf{vs AR} & \textbf{vs SchemaTV} & \textbf{TV shift} \\
\midrule
AR + xgrammar & 0 & 33.8 & 1.00$\times$ & 0.39$\times$ & --- \\
DFlash SchemaTV & 0 & 87.1 & 2.58$\times$ & 1.00$\times$ & 0 \\
DFlash + true $\Phi$ top-$4$ & yes & 60.9 & 1.81$\times$ & 0.70$\times$ & 0.069 \\
DFlash + cheap $\Phi$ top-$4$ & no & 84.7 & 2.53$\times$ & 0.98$\times$ & 0.150 \\
\bottomrule
\end{tabular}
\end{table}

\paragraph{Open-xgrammar terminal-summary diagnostic.}
We also ran an open-xgrammar smoke test that saves generated JSON texts from the
same DFlash verifier loop and analyzes coarse parsed summaries. This experiment
is deliberately \emph{not} used as fidelity evidence: the natural AR+xgrammar
reference is itself locally projected rather than $\mu^\star$, the open schemas
contain free strings and numeric values that make exact terminal-string TV
sample-inefficient, and at $N=50$ prompts/schema the observed summary shifts are
small and heterogeneous. Concretely, a trigger-calibrated top-$16$
candidate-conditioned path changes mean signature-summary TV against AR+xgrammar
from 0.229 to 0.223 and feature-marginal TV from 0.0297 to 0.0289, with
\texttt{api\_response} regressing on both summaries. We therefore treat this run
only as an implementation smoke test for text capture, matcher-state cloning,
and gated $\Phi$ plumbing in open xgrammar; the terminal-law claims in this paper
come from the finite-trie experiments where $\mu^\star$ is exactly defined.

\paragraph{Production-like terminal-law pilot.} To directly test terminal
sampling fidelity in a verifier path that uses the real DFlash draft model, we
replace open xgrammar with the finite canonical JSON token tries from
\Cref{tab:real_model}. This makes $\mu^\star$ exactly enumerable while retaining
the Qwen3-8B target and z-lab/Qwen3-8B-DFlash-b16 draft path. We draw
$N=10{,}000$ terminal samples per schema under DFlash SchemaTV and under the same
draft path with exact finite-trie $\Phi$ reweighting and the standard Leviathan
accept/reject verifier. \Cref{tab:dflash_terminal_tv} shows that exact $\Phi$
reduces terminal TV to $\mu^\star$ on all four headline schemas, from mean
0.350 to 0.017 TV, a 95\% pooled reduction, while measured throughput changes
from 91.7 to 81.1 tok/s on average. On the 2{,}000-string Schema~E, the same
correction reduces TV from 0.283 to 0.062.

This result also resolves the residual diagnosed by our earlier
equality-with-independent-target-sample verifier. That diagnostic accepts a draft
token when it equals an independent target posterior sample. For a one-step
categorical target $p$ and draft $q$, the resulting marginal is
$p_i(1 + q_i - \langle p,q\rangle)$ rather than $p_i$, so it can amplify
draft-favored tokens even after the target posterior is $\Phi$-corrected. With
that equality verifier, exact $\Phi$ left mean TV 0.088 on A--D and 0.255 on
Schema~E. Replacing it with the standard accept/reject kernel lowers those
residuals to 0.017 and 0.062, respectively. The grammar is still an enumerable
token trie, not an open xgrammar schema with arbitrary whitespace or free-text
fields, so we treat this as production-like verifier-path fidelity rather than
open-xgrammar terminal-law evidence.

\begin{table}[t]
\centering
\caption{Production-like DFlash finite-trie terminal-law TV on Qwen3-8B
($N=10{,}000$ samples/schema). The verifier uses z-lab/Qwen3-8B-DFlash-b16 and
finite canonical JSON token tries, so $\mu^\star$ is exactly known. Exact
$\Phi$ improves terminal TV on every schema. The tok/s column reports
SchemaTV/exact-$\Phi$ throughput for the same DFlash path. Schema~E is a
larger-support stress case and is not included in the A--D headline mean.}
\label{tab:dflash_terminal_tv}
\small
\setlength{\tabcolsep}{3pt}
\begin{tabular}{lccccc}
\toprule
\textbf{Schema} & \textbf{Strings} & \textbf{SchemaTV TV} & \textbf{exact-$\Phi$ TV} & \textbf{tok/s} & \textbf{Reduction} \\
\midrule
status (A) & 3 & 0.180 & 0.011 & 53.3/49.0 & 94\% \\
type-value (B) & 4 & 0.429 & 0.031 & 110.4/96.7 & 93\% \\
action-target (C) & 18 & 0.216 & 0.013 & 99.4/89.5 & 94\% \\
method-path (D) & 24 & 0.576 & 0.014 & 103.6/89.3 & 98\% \\
\midrule
Mean A--D & --- & 0.350 & 0.017 & 91.7/81.1 & 95\% \\
\midrule
flag-code (E) & 2{,}000 & 0.283 & 0.062 & 76.5/78.3 & 78\% \\
\bottomrule
\end{tabular}
\end{table}

The DFlash TV values are point estimates from terminal samples. For Schema~E,
the support has 2{,}000 strings and the ideal finite-trie FVO-Spec loop reaches
TV 0.0148 only at $N=50{,}000$; scaling that sampling floor to $N=10{,}000$
puts a substantial part of the observed 0.062 residual in the expected
finite-sample regime. We therefore use Schema~E as a stress diagnostic rather
than as the headline mean.

\paragraph{Independent-seed check.}
To check that the DFlash verifier-path gain is not a single-sample accident, we
rerun the A--D finite-trie terminal-law experiment with three additional random
seeds at $N=2{,}000$ terminal samples per schema and seed. Exact finite-trie
$\Phi$ improves every one of the 12 schema--seed pairs. Aggregating these
independent runs, mean SchemaTV TV is 0.349 and mean exact-$\Phi$ TV is 0.021,
a 94\% pooled reduction; a two-sided sign test over the 12 paired TV
differences gives $p=4.9{\times}10^{-4}$. We keep the $N=10{,}000$ run in
\Cref{tab:dflash_terminal_tv} as the headline point estimate and use the
independent-seed run as repeatability evidence.

\paragraph{Functional reading of the TV gap.} Even this three-string status
schema has an operational interpretation. Under $\mu^\star$, the Qwen3-8B
conditional distribution assigns 8.8\% mass to \texttt{"error"}; DFlash SchemaTV
emits \texttt{"error"} 26.8\% of the time, a $3.1\times$ overproduction of the
error branch. Exact finite-trie $\Phi$ with the standard verifier reduces this
to 7.9\%. This is not a
separate downstream benchmark, but it shows that the terminal-TV gap corresponds
to concrete branch-frequency shifts a tool-using system would observe, not only
to an abstract distributional metric.

\subsection{E7: Real-Model Exact $\Phi$ Correction}

To verify the $\mu^{\proj}$--$\mu^\star$ gap and its correction are not artifacts of toy LMs, we evaluate on Qwen3-8B with five finite canonical JSON languages ranging from 3 to 2{,}000 valid strings. Four are small schemas with 3--24 canonical JSON strings; the fifth is a regular code schema whose exact $\Phi$ computation requires backward DP over a 4{,}232-node token trie. We query Qwen3-8B for $p(y \mid \text{prompt}, x_{<t})$ at every trie prefix, and compute three distributions analytically: exact $\mu^\star$ by sequence likelihood, $\mu^{\proj}$ by local trie masking and renormalization, and exact-$\Phi$ by backward dynamic programming over the trie. This finite-language protocol has no sampling noise and no dependence on matcher rollback or whitespace termination behavior.

\begin{table}[t]
\centering
\caption{Real-model exact-$\Phi$ correction on Qwen3-8B with finite canonical JSON languages. Local masking misallocates substantial mass, while exact token-trie $\Phi$ correction recovers $\mu^\star$ to numerical precision. Schema~E is a regular finite language with 2{,}000 strings.}
\label{tab:real_model}
\small
\begin{tabular}{lcccc}
\toprule
\textbf{Schema} & \textbf{Strings} & $\TV(\mu^{\proj}, \mu^\star)$ & $\TV(\mu^\Phi, \mu^\star)$ & \textbf{Key Distortion} \\
\midrule
status (A) & 3 & 0.188 & $4.3{\times}10^{-16}$ & $\mu^{\proj}$ overweights ``error'' (27\% vs 8\%) \\
type-value (B) & 4 & 0.421 & $4.8{\times}10^{-16}$ & $\mu^{\proj}$ overweights type=b by $26\times$ \\
action-target (C) & 18 & 0.174 & $1.0{\times}10^{-15}$ & target/async mass shifted across fields \\
method-path (D) & 24 & 0.681 & $1.3{\times}10^{-15}$ & $\mu^\star$: GET/metrics; $\mu^{\proj}$: GET/api \\
flag-code (E) & 2{,}000 & 0.244 & $7.0{\times}10^{-16}$ & $\mu^{\proj}$ overweights code 000 by $12\times$ \\
\bottomrule
\end{tabular}
\end{table}

\Cref{tab:real_model} reports the results across all five schemas. The local-projection gap ranges from 0.174 to 0.681 TV on Qwen3-8B, confirming that the bias exists on a real 8B transformer and can be large even for finite JSON languages. Exact finite-language $\Phi$ correction closes the gap in every schema, with maximum residual TV $<2{\times}10^{-15}$. The distortion is qualitative as well as quantitative: on Schema~B, $\mu^\star$ assigns only 1.7\% mass to type=b strings, while $\mu^{\proj}$ assigns 43.8\%; on Schema~D, $\mu^\star$ concentrates on GET/metrics/v1 (55.4\%) while $\mu^{\proj}$ favors GET/api/v1 (62.6\%); on Schema~E, $\mu^{\proj}$ assigns 12.1\% to flag=true/code=000 while $\mu^\star$ assigns 1.0\%. As a prompt-robustness check, we repeat the exact-token-trie experiment with three prompt wordings on Schemas~A--D; all 12 prompt--schema pairs close the gap, with mean projected TV 0.405, minimum projected TV 0.119, and maximum exact-$\Phi$ residual TV $1.5{\times}10^{-15}$.

\begin{table}[t]
\centering
\caption{Finite-trie FVO-Spec loop on Qwen3-8B prefix probabilities (draft block $\gamma=4$; $N=50{,}000$ for all rows). The loop performs speculative accept/reject sampling against the $\Phi$-reweighted target over the same token tries as \Cref{tab:real_model}; it is an online sampler validation, not a production xgrammar/DFlash throughput benchmark.}
\label{tab:fvo_loop}
\small
\begin{tabular}{lccc}
\toprule
\textbf{Schema} & $\TV(\mu^{\proj}, \mu^\star)$ & $\TV(\hat{\mu}^{\mathrm{FVO}}, \mu^\star)$ & \textbf{Accept rate} \\
\midrule
status (A) & 0.188 & 0.0010 & 0.955 \\
type-value (B) & 0.421 & 0.0017 & 0.932 \\
action-target (C) & 0.174 & 0.0029 & 0.978 \\
method-path (D) & 0.681 & 0.0052 & 0.922 \\
flag-code (E) & 0.244 & 0.0148 & 0.966 \\
\midrule
Mean & 0.342 & 0.0051 & 0.951 \\
\bottomrule
\end{tabular}
\end{table}

\Cref{tab:fvo_loop} verifies that the correction is not merely an analytic identity. We instantiate the FVO-Spec accept/reject loop with a locally projected draft and exact token-trie $\Phi$, then compare the empirical terminal-string law to $\mu^\star$. On Schemas~A--D, the sampled law is within TV $0.0052$ of $\mu^\star$; on the 2{,}000-string Schema~E it is within TV $0.0148$ at $N=50{,}000$. The higher residual on Schema~E reflects finite-sample estimation noise over a 2{,}000-atom support, not algorithm error: the analytic $\Phi$-TV is $<2{\times}10^{-15}$, and the residual is consistent with the expected $O(\sqrt{|\calL|/N})$ sampling floor. The uncorrected locally projected law is off by mean TV $0.342$ across the five schemas. This real-model experiment directly validates the Doob-transform correction: the same Qwen3-8B logits that produce the biased locally-masked distribution recover $\mu^\star$ once reweighted by exact future validity.

\end{document}